\pdfoutput=1
\documentclass{article}
\usepackage{iclr2027_conference,times}

\usepackage[utf8]{inputenc}
\usepackage[T1]{fontenc}
\usepackage{amsmath}
\usepackage{hyperref}
\usepackage{url}
\usepackage{cleveref}
\usepackage{booktabs}
\usepackage{amsthm}
\usepackage{amsfonts}
\usepackage{amssymb}
\usepackage{multirow}
\usepackage{nicefrac}
\usepackage{microtype}
\usepackage{xcolor}
\usepackage{xspace}
\usepackage{graphicx}
\usepackage{subcaption}
\usepackage{wrapfig}
\usepackage{algorithm}
\usepackage{algpseudocode}

\theoremstyle{plain}

\newtheorem{proposition}{Proposition}

\makeatletter
\newcommand*{\rom}[1]{\expandafter\@slowromancap\romannumeral #1@}
\makeatother

\def\pdata{\ensuremath{p_{\mathrm{data}}}}

\def\E{\mathbb E}

\def\mask{\mathbf{m}}

\def\param{\theta}

\renewcommand{\vec}[1]{\mathbf{#1}}

\DeclareMathOperator{\KL}{KL}

\newcommand{\rrd}{\ensuremath{\mathrm{R}^2\mathrm{D}}\xspace}
\newcommand{\std}[1]{{\scriptstyle \pm #1}}

\title{Random Remasking Scales \\ Masked Diffusion Inference}

\author{Zhaoyi Zhou \quad Frank Wu \quad Andrea Zanette \\
Carnegie Mellon University \\
\texttt{\{zhaoyiz,frankwu2,azanette\}@andrew.cmu.edu}
}

\iclrfinalcopy

\renewcommand{\headrulewidth}{0pt}

\makeatletter
\def\@maketitle{\vbox{\hsize\textwidth\centering
{\LARGE\sc \@title\par}
\begin{tabular}[t]{c}\bf\rule{\z@}{24pt}\@author\end{tabular}%
\vskip 0.3in minus 0.1in}}
\makeatother

\begin{document}

\maketitle
\fancyhead{}

\begin{abstract}
Masked diffusion models (MDMs) generate sequences by progressively unmasking tokens, but errors introduced during decoding are typically irreversible. Recent work has shown that iterative correction can improve sample quality, yet existing theory fails to explain this given a learned, imperfect denoiser model. We re-formulate MDM inference as a Markov chain on the clean sequence space, defined by a remask-denoise kernel that randomly re-introduces masks and then denoises. This helps us analyze the actual denoising trajectory towards the data distribution. The analysis motivates Random Remask-Denoise (\rrd) inference, a purely test-time, finetuning-free, reward-free procedure. To further improve performance, we propose Trajectory Correction Training (TCT), a denoiser-training procedure that exposes the model to its own errors along the generation trajectory. On logical tasks, our approach exhibits a striking near-log-linear test-time scaling, achieving over 97\% accuracy on hard Sudoku puzzles and over 99\% accuracy on 3-SAT. Our approach also applies to pretrained models on actual language tasks, beating iterative correction baselines and the official Q-mode accuracy of LLaDA2.1-mini on HumanEval+.
\end{abstract}

\section{Introduction}

Masked diffusion models (MDMs) have emerged as a strong alternative to autoregressive models (ARMs) for sequence generation \citep{sahooSimpleEffectiveMasked2024,lou2024discrete,ou2024your,shi2024simplified}. By denoising under bidirectional attention rather than committing to a left-to-right factorization, MDMs can in principle leverage context from both sides when filling in tokens, an advantage that has shown promise on reasoning-heavy tasks.

A central limitation of standard MDM decoding is that revealed tokens are frozen for the rest of the trajectory: Errors introduced by parallel decoding cannot be undone. Iterative correction methods have been proposed to resolve this, either by remasking committed tokens for re-generation \citep{wangRemaskingDiscreteDiffusion2026}, or by training a model with self-correction ability \citep{schiffLearnYourMistakes2026a}.
These methods not only serve as a remedy for parallel-decoding errors, but also unlock \emph{test-time scaling} for MDMs: The number of model evaluations is no longer bounded by the response length. Each correction step can be viewed as an additional unit of compute spent refining the answer, in spirit analogous to chain-of-thought reasoning for autoregressive models, with iterative error correction itself acting as a form of reasoning.

Despite the empirical success, iterative-correction methods fail to explain why they improve generation quality from the theoretical perspective.
Existing explanations, e.g., the ReMDM \citep{wangRemaskingDiscreteDiffusion2026} and the predictor--corrector framework~\citep{campbellContinuousTimeFramework2022} only provide convergence analysis of the intermediate sequence marginal distribution under a pefect denoising model. They fail to analyze the actual trajectory generated by an imperfect denoiser, which is the question that really matters.

\begin{figure}[t]
  \centering
  \begin{subfigure}[t]{0.49\linewidth}
    \centering
    \includegraphics[width=\linewidth]{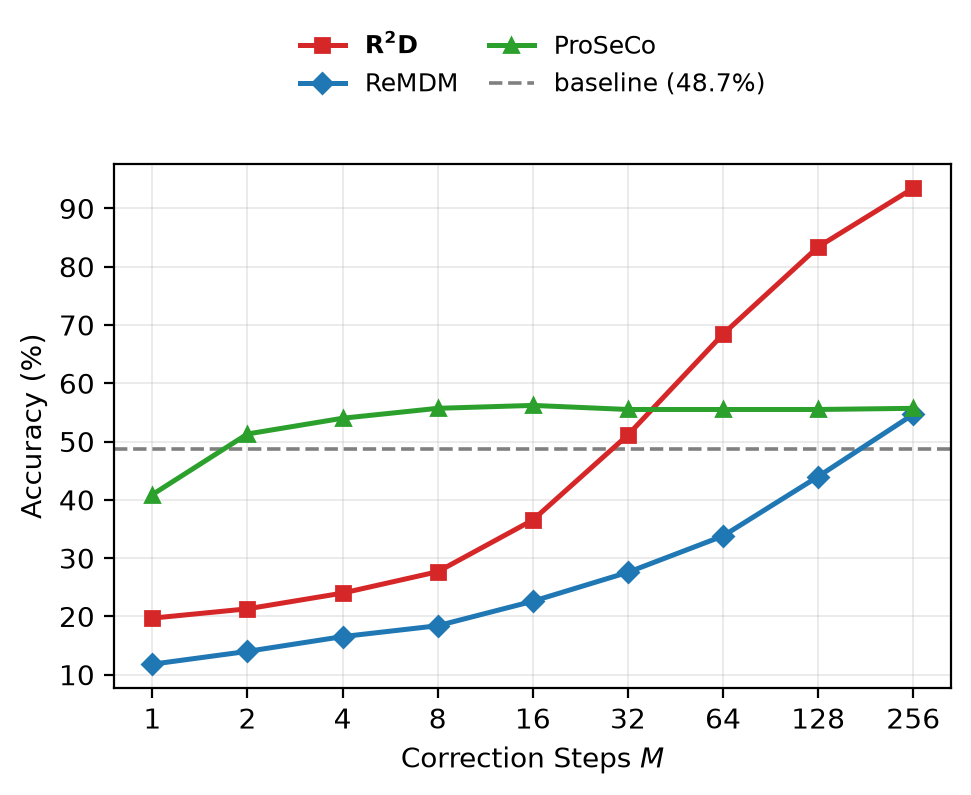}
    \caption{Sudoku-Hard}
  \end{subfigure}\hfill
  \begin{subfigure}[t]{0.49\linewidth}
    \centering
    \includegraphics[width=\linewidth]{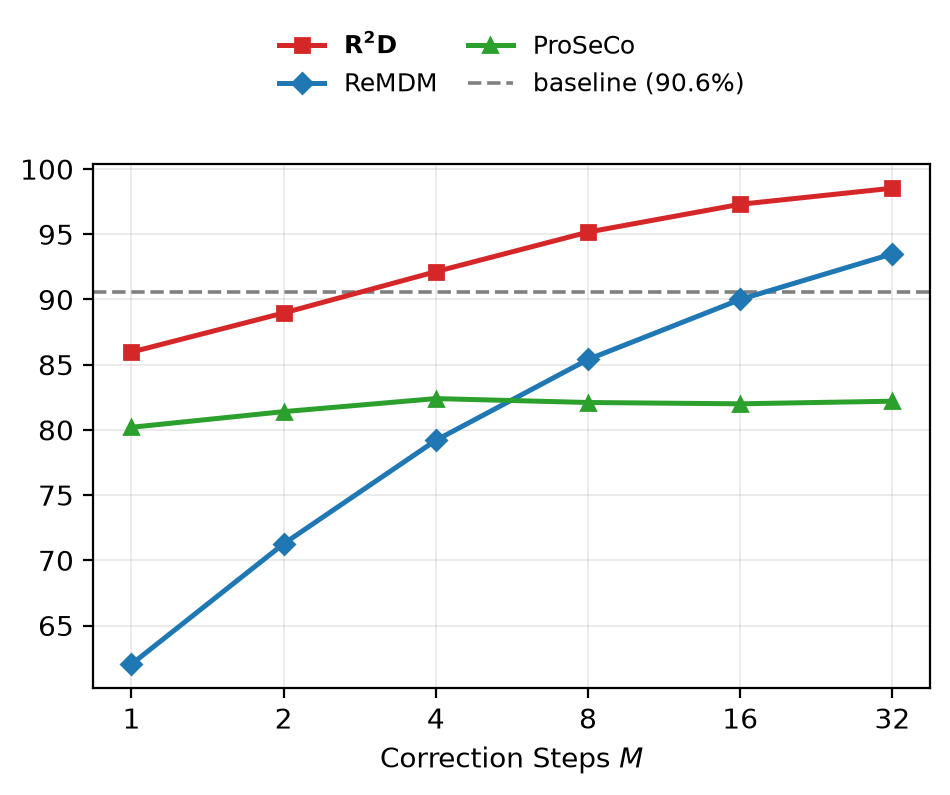}
    \caption{3-SAT}
  \end{subfigure}
  \caption{Inference scaling on Sudoku-Hard and 3-SAT. \rrd and ReMDM are run on standard training checkpoints, both using the standard variant; ProSeCo has its own training method. \rrd scales near-log-linearly on both tasks and stays above ReMDM at every budget, whereas ProSeCo plateaus after a handful of correction steps. }
  \label{fig:main-scaling}
\end{figure}

The prevailing formulation treats MDM inference as an iterative denoising process from a fully masked seuqence into a clean sequence. This is hard for a generation trajectory analysis, as intermediate sequences correspond to different masking ratios and hence lie in different subspaces. Observing that each inference step can be further decomposed into a \emph{denoise} step and a \emph{remask} step, we re-formualte the entire MDM inference process in this work as an iteration over the \emph{clean sequence space}: At each step, some of the tokens in the clean sequence are remasked, and then denoised back into a clean sequence. We identify the \emph{remask--denoise kernel} as the single object underlying both standard decoding and correction, and we derive a single-step update bound of the KL divergence that identifies how the chain more likely moves toward the data distribution at each step.

The bound exposes three factors: Inference timestep schedule, task structure and model quality. The bound suggests a  non-increasing timestep schedule, reducing gradually from 1 (fully masked sequence) to 0 (clean sequence). The challenge is how to decide the number of iterations at each step. We propose a natural solution: Keep iterating at the given step until the clean sequence prediction stabilizes, which leads to our \textbf{Random Remask-Denoise} (\rrd) inference. Random remasking is more favorable for tasks with local constraints, so that the model better reconstructs masks with revealed information, leading to lower denoising error. The model should be trained under the actual inference trajectory, and thus we propose \textbf{Trajectory Correction Training} (TCT), a self-correction training method that is paired with \rrd inference.

On logical tasks like Sudoku-Hard and 3-SAT, \rrd with standard training shows significant scaling effect (\Cref{fig:main-scaling}), boosting accuracy from 19.7\% and 86\% to 93.5\% and 98.5\%. TCT training further increases accuracy to 97.5\% and 99.3\%, respectively. \rrd also works well on language tasks. On the coding benchmark HumanEval+, \rrd with off-the-shelf LLaDA2.1-mini surpasses iterative correction baselines and the official Q-mode inference.

Our main contributions are:
\begin{itemize}
  \item A unified Markov-chain framework for MDM inference, in which standard decoding and self-correction are instances of a single remask--denoise kernel acting on the complete sequence space.
  \item Random Remask-Denoise (\rrd), a reward-free, finetuning-free inference algorithm with an adaptive timestep schedule motivated by the single-step update bound.
  \item Trajectory Correction Training (TCT), a denoiser-training procedure that supplies trajectory-aware supervision at every position and pairs naturally with \rrd.
  \item Across domains like logical tasks and coding, we demonstrate \rrd achieves test-time scaling, and the scaling effect is dependent on the token coupling nature of the task. TCT training further boosts \rrd with self-correction ability.
\end{itemize}

\section{Related work}
\paragraph{Iterative-correction inference.}
A growing body of work equips MDMs with the ability to revise tokens that have already been revealed, falling into two groups. \emph{Training-free} methods \citep{wangRemaskingDiscreteDiffusion2026, shahOSCAROrchestratedSelfverification2026} modify only the inference procedure of a pretrained denoiser. For instance, ReMDM \citep{wangRemaskingDiscreteDiffusion2026} derives a principled remasking sampler from a custom backward process and remasks already-decoded tokens either uniformly at random or by model confidence before re-predicting them. \emph{Trained} approaches instead fine-tune or augment the denoiser to recognize and revise its own mistakes. ProSeCo~\citep{schiffLearnYourMistakes2026a} learns a corrector that consumes the denoiser's own outputs and alternates unmasking with correction at inference. PRISM~\citep{kimFineTuningMaskedDiffusion2025}, RemeDi~\citep{huangDontSettleToo2025}, BackPlay~\citep{liuBackPlayPluginLookBack2026}, and Corrective DLM~\citep{zhangCorrectiveDiffusionLanguage2026} differ in the training recipe but share the same goal of producing per-token quality signals that drive selective remasking. All of these approaches frame correction defensively, as a way to repair errors introduced by parallel decoding; We instead treat iterative remasking and correction as a reasoning mechanism in its own right and as a primary axis for test-time scaling.

\paragraph{Test-time scaling for reward guidance.} Most test-time scaling work for MDMs lies in the reward-guidance regime, in which an external reward function or verifier steers sampling toward higher-quality outputs. These methods scale through three mechanisms: \emph{particle-based sampling}~\citep{singhalGeneralFrameworkInferencetime2025,dangInferenceTimeScalingDiffusion2026,ouInferenceTimeScalingDiscrete2026}, \emph{iterative noise--denoise refinement}~\citep{leeEffectiveTestTimeScaling2025,ueharaRewardGuidedIterativeRefinement2025,phunyaphibarnRewardGuidedDiscreteDiffusion2026}, and \emph{direct guidance}~\citep{tejaswiEntRGiEntropyAware2026,chenRFGTestTimeScaling2025}. The closest to our framing is \citet{phunyaphibarnRewardGuidedDiscreteDiffusion2026}, which also samples from a Markov chain on the complete sequence space but offers limited theoretical analysis of the sampler's convergence. All three families rely on an external quality signal to guide intermediate generation, whereas the inference algorithms we study are reward-free.

\section{Preliminaries}
We briefly review the absorbing masked diffusion model (MDM) setup that our method builds on.

\paragraph{Setup.}
Let $\vec{x}_0 = (\vec{x}_0^1, \dots, \vec{x}_0^L) \in \mathcal{V}^L$ be a clean discrete sequence of length $L$, where $\mathcal{V}$ is a finite vocabulary. We augment the vocabulary with a special mask token $\mask$ and write $\vec{x}_t \in (\mathcal{V} \cup \{\mask\})^L$ for the corrupted sequence at diffusion time $t \in [0,1]$. Let $\alpha_t \in [0,1]$ be a monotonically decreasing noise schedule with $\alpha_0 = 1$ and $\alpha_1 = 0$. Clean sequences are drawn from a data distribution $\pdata$ over $\mathcal{V}^L$.

\paragraph{Diffusion training objective.}
An MDM learns a network $f_\theta$ whose token-wise output $f_\theta^i(\cdot \mid \vec{x}_t) \in \Delta(\mathcal{V})$ predicts the clean token given the current intermediate sequence $\vec{x}_t$. Standard MDMs adopt carry-over unmasking, i.e., $f_\theta^i(\vec x_t^i \mid \vec{x}_t)=1$ if $\vec x_t^i\in \mathcal V$, while self-correction models relax this constraint. The training objective simplifies to a weighted cross-entropy loss over masked tokens \citep{ou2024your,sahooSimpleEffectiveMasked2024,shi2024simplified}:
\begin{equation*}
  \mathcal{L}_{\mathrm{MDM}}(\theta)
  = \mathbb{E}_{\vec{x}_0 \sim \pdata,\, t,\, \vec{x}_t \sim q_{t\mid 0}(\cdot \mid \vec{x}_0)}
  \left[
    \frac{\dot{\alpha}_t}{1-\alpha_t}
    \sum_{i=1}^L \mathbf{1}\{\vec{x}_t^i = \mask\}
    \bigl(-\log f_\theta^i(\vec{x}_0^i \mid \vec{x}_t)\bigr)
  \right].
\end{equation*}
Standard MDM training thus teaches the model to recover each clean token independently at every currently masked position. Because the per-position outputs are conditionally independent given $\vec x_t$, the network induces a joint distribution
\begin{equation*}
  f_\theta(\vec x_0 \mid \vec x_t) \;:=\; \prod_{i=1}^L f_\theta^i(\vec x_0^i \mid \vec x_t),
\end{equation*}
which we use in \Cref{sec:method} as the model's posterior over clean sequences.

\paragraph{Standard diffusion inference.}
At inference time, generation starts from the fully masked sequence $\vec{x}_{t_0} = (\mask, \dots, \mask)$ and proceeds along a decreasing time grid $1=t_0 > t_1 > \dots > t_K = 0$. At each step, the model predicts token distributions $f_\theta^i(\cdot \mid \vec{x}_{t_j})$ at all masked positions, and an unmasking policy selects a subset $S_j \subseteq \{i : \vec{x}_{t_j}^i = \mask\}$ of size $(\alpha_{t_{j+1}} - \alpha_{t_j})L$ to reveal --- either uniformly at random or by model confidence. For each $i \in S_j$, a token is drawn from $f_\theta^i(\cdot \mid \vec{x}_{t_j})$ (by sampling or arg-max) and written into the sequence to produce $\vec{x}_{t_{j+1}}$. Standard inference never revises an already revealed token: Once a position is unmasked, it stays fixed for the remainder of the trajectory. Our work departs from this convention by allowing iterative correction of previously revealed tokens.

\section{Method}
\label{sec:method}

The standard view of MDM inference is an iteration over the intermediate sequence space: Starting at $t = 1$, the noise level decreases monotonically until a clean sequence $\vec x_0$ is produced. This view is faithful to the underlying diffusion theory but has two limitations. First, correction is treated as an engineering add-on for fixing model errors rather than an inherent part of inference. Second, the framework makes it difficult to reason about iterative correction under an imperfect denoiser: Because the iteration variable $\vec x_t$ changes type at every step, intermediate sequences live in different spaces, and there is no unified metric for tracking how close the current state is to the data distribution.

Observe that each inference step can be further decomposed into a \emph{denoise} step and a \emph{remask} step. In the denoise step, the model predicts the clean sequence $\vec x_0$ from any intermediate sequence $\vec x_t$, and in the remask step, some of the tokens in $\vec x_0$ are turned back into masks, resulting in another intermediate sequence at another noise level $\vec x_{t'}$. In other words, the MDM inference process, either standard denoising or iterative correction, can be unified into the following transition:
\begin{align*}
  & \vec x_{t_0}\to \vec x_0 \to \vec x_{t_1} \to \vec x_0 \to \cdots \to \vec x_{t_n} \\
  & t_0=1, t_n=0
\end{align*}
Therefore, we can view the inference process as an iteration over the clean sequence space $\mathcal V^L$. Under a Markovian remask step, the transition is given by a \emph{remask--denoise kernel}:
\begin{align*}
  \tilde K^\theta_{t}(\vec x_0' \mid \vec x_0) \;=\; \sum_{\vec x_t} f_\theta(\vec x_0' \mid \vec x_t)\,q_{t\mid 0}(\vec x_t \mid \vec x_0),
\end{align*}
which approximates the Bayes-optimal counterpart
\begin{align*}
  K_{t}(\vec x_0' \mid \vec x_0) \;=\; \sum_{\vec x_t} q_{0\mid t}(\vec x_0' \mid \vec x_t)\,q_{t\mid 0}(\vec x_t \mid \vec x_0).
\end{align*}
We refer to $\tilde K^\theta_t$ as the \emph{model kernel} and to $K_t$ as the \emph{true kernel}.

\begin{proposition}[Stationarity of the true kernel, proof in \Cref{app:stationarity-proof}.]
\label{prop:stationarity}
For every $t \in (0,1)$, the kernel $K_t$ admits $\pdata$ as a stationary distribution: $\pdata K_t = \pdata$.
\end{proposition}
\Cref{prop:stationarity} implies that the model kernel $\tilde K^\theta_t$ must be close to the true kernel $K_t$ for the chain to converge near $\pdata$, which in turn requires the remasking step to be strictly Markovian. However, the remask step in standard diffusion inference violates this requirement, since it only remasks a previous mask, which is history information inaccessible in $\vec x_0$. In our formulation non-Markovian remasking is an unnecessary source of error, and thus we recommend remasking every token in the clean sequence $\vec x_0$ uniformly at random, ensuring the Markov property.

This convergence guarantee, however, is hardly possible to achieve with a learned denoiser. So we would like to know how to make the actual generation close to the data distribution, by doing the best in every single step.
To analyze it, we fix any $p \neq \pdata$ over $\Delta(\mathcal V^L)$ on the chain's state space and write $p' = p \tilde K^\theta_t$ for its image after one kernel step, where $(p\tilde K^\theta_t)(\vec x_0') := \sum_{\vec x_0} p(\vec x_0)\,\tilde K^\theta_t(\vec x_0' \mid \vec x_0)$. We use $\KL(\pdata \| p)$ as the discrepancy metric between the ground truth distribution $\pdata$ and current distribution $p$.

\begin{proposition}[Single-step update bound of KL divergence]\label{prop:identity}
For any distribution $p \neq \pdata$ over $\mathcal{V}^L$ and kernel step $t$, let $\tilde{q}_t(\vec x_t) = \sum_{\vec x_0} q_{t\mid 0}(\vec x_t \mid \vec x_0)\,p(\vec x_0)$ be the marginal obtained by remasking $p$, and define the \emph{contraction ratio} of the masking channel
\begin{equation}\label{eq:rho}
\rho_t(p) \;=\; \frac{\KL(q_t \,\|\, \tilde{q}_t)}{\KL(\pdata \,\|\, p)},
\end{equation}
which satisfies $\rho_t(p) \in [0, 1]$ by the data-processing inequality. Then
\begin{equation*}
\KL(\pdata \,\|\, p') \;\le\; \rho_t(p) \cdot \KL(\pdata \,\|\, p) \;+\; \epsilon_t,
\end{equation*}
where
\begin{equation*}
\epsilon_t \;=\; \E_{\vec x_t \sim q_t}\bigl[\KL\bigl(q_{0\mid t}(\cdot \mid \vec x_t) \,\big\|\, f_\theta(\cdot \mid \vec x_t)\bigr)\bigr].
\end{equation*}
\end{proposition}
The proof, which derives the stronger equality form, is in \Cref{app:identity-proof}. The two terms in \Cref{prop:identity} have distinct interpretations:
\begin{itemize}
  \item $\rho_t(p)$, the \emph{contraction ratio}, is monotonically decreasing in $t$: A larger $t$ destroys more information about $\vec x_0$, making the two marginals more similar (data-processing inequality), so the term shrinks smoothly from 1 to 0 as $t$ ranges from $0$ to $1$.
  \item $\epsilon_t$ is the \emph{denoising error}: The expected KL divergence between the Bayes-optimal posterior $q_{0\mid t}(\cdot \mid \vec x_t)$ and the model's $f_\theta(\cdot \mid \vec x_t)$. At large $t$ the input is mostly masked and prediction is harder, so $\epsilon_t$ typically increases with $t$.
\end{itemize}

We want to achieve a small $\KL(\pdata \,\|\, p')$ at each step, and three factors affect the update bound: Timestep schedule, task structure, and model quality. We analyze each in turn in the following subsections.

\subsection{Schedule: Random Remask-Denoise (\rrd) Inference}

\begin{algorithm}[t]
\caption{Random Remask-Denoise (\rrd) Inference (standard \& correction variants).}
\label{alg:r2d-inference}
\begin{algorithmic}[1]
\Require Trained MDM $f_\theta$, decoding steps $K$, per-level correction budget $M$, stability threshold $\tau$
\State $\vec x_1 \gets \mask^L$ \Comment{initialize fully masked, $t = 1$}
\For{$i = K-1, K-2, \cdots, 0$}
    \State $t \gets i / K$
    \State $\hat{\vec x}_0^{\mathrm{prev}} \gets \texttt{None}$;\quad $c \gets 0$ \Comment{consecutive-repetition counter}
    \For{$j = 1, 2, \cdots, M+1$} \Comment{repeat denoise-remask at noise level $t$}
        \State $\hat{\vec x}_0 \gets \arg\max f_\theta(\cdot \mid \vec x_t)$
        \If{\textbf{standard} variant} \Comment{skipped by the \emph{correction} variant}
            \State $(\hat{\vec x}_0)_\ell \gets (\vec x_t)_\ell$ for $\ell$ with $(\vec x_t)_\ell \neq \mask$ \Comment{freeze revealed tokens}
        \EndIf
        \State $c \gets c + 1$ \textbf{if} $\hat{\vec x}_0 = \hat{\vec x}_0^{\mathrm{prev}}$ \textbf{else} $0$
        \State $\hat{\vec x}_0^{\mathrm{prev}} \gets \hat{\vec x}_0$
        \State $n \gets \lfloor (1-t)\, L \rfloor$ \Comment{\# positions to reveal at noise level $t$}
        \State sample $S \subseteq \{1, \cdots, L\}$ uniformly with $|S| = n$
        \State $(\vec x_t)_\ell \gets (\hat{\vec x}_0)_\ell$ for $\ell \in S$;\quad $(\vec x_t)_\ell \gets \mask$ for $\ell \notin S$ \Comment{remask}
        \If{$c \geq \tau$} \Comment{converged: $\tau+1$ identical predictions in a row}
            \State \textbf{break}
        \EndIf
    \EndFor
\EndFor
\State \Return $\hat{\vec x}_0$
\end{algorithmic}
\end{algorithm}
The next question is how to choose the timestep $t$ at each step of the chain. A non-increasing schedule is intuitively right: Early on $\KL(\pdata\|p)$ is large and a larger $t$ can reduce the factor $\rho_t(p)$ even at the cost of higher denoising error $\epsilon_t$, whereas late in the chain $\KL(\pdata\|p)$ is small and only a smaller $t$ keeps $\epsilon_t$ from dominating. ProSeCo \citep{schiffLearnYourMistakes2026a} skips remasking entirely during correction (i.e.\ $t=0$), which is suboptimal under \Cref{prop:identity} since $\rho_0(p) = 1$ and does not enjoy contraction. Standard inference and ReMDM \citep{wangRemaskingDiscreteDiffusion2026} follow this intuition, but the latter uses a fixed uniform schedule by default, which is suboptimal: Early large-$t$ steps are wasted while $\rho_t(p)\KL(\pdata\|p)$ still dominates, and late small-$t$ steps are bottlenecked by $\epsilon_t$ and yield little further progress. Heuristically, we want to move on to a smaller timestep once performance saturates.

We therefore propose \textbf{Random Remask-Denoise} (\rrd) inference, given in \Cref{alg:r2d-inference}, which uses an adaptive timestep schedule. The noise variable is discretized into $K$ equally spaced grid points $t \in \{(K-1)/K, (K-2)/K, \cdots, 1/K, 0\}$, with $K = L$ as a natural default (one outer step per token). Starting from the largest $t$, we hold the noise level fixed and iterate the model kernel $\tilde K^\theta_t$, treating the inner loop as a search for a fixed point of $\tilde K^\theta_t$ at that level.

The main challenge lies in performance saturation detection. KL divergence is our theoretical metric, but it is intractable to compute as it involves the actual distribution of the intermediate sequence. We turn to a heuristic alternative: Under argmax sampling, convergence can be detected once $\hat{\vec x}_0$ has remained unchanged across $\tau$ consecutive denoise steps (Line 15). We set a per-level budget $M$ so that inference ends with at most $K(M+1)$ forward passes, but our experiments show that actual inference uses far fewer forward passes.

We instantiate \rrd in two variants depending on whether the denoiser has been trained for correction, differing only in Line 7 of \Cref{alg:r2d-inference}. The \emph{standard} variant executes that line, carrying over the revealed positions of $\vec x_t$ and filling in only the masked entries, which matches the behavior assumed by an off-the-shelf MDM denoiser. The \emph{correction} variant skips it, overwriting every position with the model's prediction so that previously revealed tokens are revised whenever the denoiser disagrees with them. In either case, \rrd is a purely test-time procedure that requires no retraining and applies to any pretrained MDM denoiser, with or without explicit correction capability. Note that the standard variant also has error correction ability through random remasking.

\subsection{Task Structure: When Is the Denoising Error Small?}
\label{sec:task}
At a fixed noise level, iterating \Cref{prop:identity} drives the KL divergence toward a floor determined by the denoising error. The error depends on how well the model $f_\theta$ learns the posterior $q_{0\mid t}$. Training is important, but task structure also plays a role here. Tasks with sparse, local constraints are easier for the model to fit the posterior. As long as model understands the constraints during training, reconstruction becomes easy with only a subset of the tokens revealed.

Sudoku and 3-SAT are two examples that fall into this regime. Every Sudoku constraint spans only 9 cells, so a cell is determined once the other 8 cells of its row, column, or box are revealed; Every 3-SAT clause has only 3 literals, so the third is forced once the other two are revealed false. A small $\epsilon_t$ at moderate $t$ in turn lets \rrd operate at larger noise levels, where $\rho_t(p)$ is smaller, before the denoising error dominates.

\subsection{Model Quality: Trajectory Correction Training (TCT)}

Having fixed the schedule, we now turn to the denoiser. Bounding $\KL(\pdata \| p')$ directly leads to a trajectory-aware ELBO over the denoiser parameters $\theta$.
\begin{proposition}[Trajectory-aware training bound]\label{prop:training-bound}
For any distribution $p$ over $\mathcal{V}^L$ and timestep $t$, let $p' = p \tilde K^\theta_t$ and $\tilde{q}_t(\vec x_t) = \sum_{\vec x_0} q_{t\mid 0}(\vec x_t \mid \vec x_0)\, p(\vec x_0)$. Then
\begin{align*}
\KL(\pdata \,\|\, p') \;\leq\; \E_{\vec x_t \sim \tilde{q}_t,\, \vec x_0 \sim \pdata}\!\bigl[-\log f_\theta(\vec x_0 \mid \vec x_t)\bigr] - H(\pdata),
\end{align*}
where $H(\pdata) = -\E_{\vec x_0 \sim \pdata}[\log \pdata(\vec x_0)]$ is the (Shannon) entropy of the data distribution.
\end{proposition}
Since $H(\pdata)$ is independent of $\theta$, minimizing the bound reduces to a cross-entropy objective on inputs $\vec x_t$ drawn from $\tilde{q}_t$ --- the marginal of the chain's own state $p$ remasked to level $t$. The proof is in \Cref{app:training-bound-proof}.

Faithfully emulating the inference trajectory of \Cref{alg:r2d-inference} during training would spend many gradient steps on a single example before the chain advances. Inspired by progressive-unmasking training (PUMA) \citep{kimStopTrainingWorst2026}, we make three simplifications:
\begin{itemize}
  \item Omit the correction inner loop, keeping only the outer decoding steps;
  \item Reduce the decoding trajectory to a fixed stage budget $K$;
  \item Use a low-confidence reveal policy to maintain generation quality under that limited budget.
\end{itemize}

The resulting procedure is Trajectory Correction Training, whose pseudo-code is given in \Cref{alg:correct-training} (\Cref{app:tct-algorithm}).
It maintains $B$ independent decoding chains in parallel, with chain $j$ at stage $n^{(j)} \in [0, K]$ of an unmasking trajectory targeting the paired clean sample $\vec x_0^{(j)}$. Unlike PUMA, which trains on the standard MDM loss, our objective is the \emph{correction loss}: Token-level cross-entropy of $f_\theta$ against the clean target evaluated at \emph{every} position, masked or already revealed,
\begin{align}
L_{\mathrm{correct}}(\param) \;=\; -\frac{1}{B} \sum_{j=1}^B \sum_{i=1}^L \log f_\param^{\,i}\!\bigl(x_0^{(j),i} \,\big|\, \tilde{\vec x}^{(j)}\bigr). \label{eq:correction-loss}
\end{align}
Including the revealed positions is what makes this a correction objective rather than the usual MDM denoising one: Tokens that the model previously filled in incorrectly contribute a non-trivial loss, training $f_\param$ to override its own past mistakes instead of merely copying its own outputs.

\section{Experiments}
\label{sec:experiments}

Our experiments are designed to answer three questions:
\begin{itemize}
  \item How do \rrd and TCT compare against iterative correction baselines? (\Cref{sec:logical,sec:humaneval})
  \item Does random remasking enable test-time scaling? (\Cref{sec:exp-scaling})
  \item How efficient is \rrd inference compared with other iterative correction approaches? (\Cref{sec:exp-compute})
\end{itemize}

\paragraph{Inference approaches.} We evaluate \rrd in both its standard and correction variants, referred to as ``\rrd (standard)'' and ``\rrd (correction)'', respectively, alongside two baselines: ReMDM~\citep{wangRemaskingDiscreteDiffusion2026} and ProSeCo~\citep{schiffLearnYourMistakes2026a}. Both baselines admit an arbitrarily large compute budget and have been reported to exhibit test-time scaling. ProSeCo inference requires its own training method, while ReMDM works with any existing model. In addition, we report the accuracy of other correction training (RemeDi \citep{huangDontSettleToo2025}, PRISM \citep{kimFineTuningMaskedDiffusion2025}, BackPlay \citep{liuBackPlayPluginLookBack2026}) and non-correction training (Standard, PUMA \citep{kimStopTrainingWorst2026}) variants.

\paragraph{Implementation variants of ReMDM.} When evaluating ReMDM with models trained with TCT, we use argmax sampling instead of the official stochastic sampling to keep the best performance. In analogy to \rrd, we use ``ReMDM (correction)'' to refer to the argmax sampling variant, and ``ReMDM (standard)'' for the original stochastic categorical sampling implementation that is used for standard-trained models. We include a detailed comparison between different ReMDM configurations in \Cref{sec:app_remdm}.

\subsection{Logical Tasks}
\label{sec:logical}
We first evaluate on two logical reasoning tasks. One is \textbf{Sudoku-Hard}, constructed from the training split of Sudoku-Extreme \citep{wang2025hierarchicalreasoningmodel}, retaining only puzzles with rating $> 10$, which yields 2.3M instances. The last 1{,}000 form the test set and the rest form the training set. Sudoku difficulty is measured by \emph{rating}, i.e., the number of backtracks needed by the tdoku solver \citep{dillion2025tdoku}. These puzzles are challenging even for experienced human solvers and dedicated solver programs, requiring substantial trial-and-error search.
The model in both tasks is a $6.8$M Qwen2-style decoder-only transformer with hidden size $256$ and $8$ layers and is trained from scratch.

\begin{table}[ht]
  \centering
  \caption{Accuracy (in $\%$) on Sudoku-Hard and 3-SAT, organized by training method and inference method. ReMDM, \rrd and ProSeCo use the same compute budget (256 steps per token for Sudoku-Hard, 32 steps per token for 3-SAT). \rrd is the primary driver for inference-time scaling, and TCT cooperates well with \rrd to further raise performance. }
  \label{tab:main-comparison}
  \begin{tabular}{llcc}
    \toprule
    Training method & Inference method & Sudoku-Hard & 3-SAT \\
    \midrule
    \multirow{4}{*}{Standard}
        & Standard                   & $48.7$ & $90.6$ \\
        & ReMDM (standard)           & $54.7$ & $93.5$ \\
        & \rrd (standard, $M=1$)     & $19.7$ & $86.0$ \\
        & \rrd (standard)   & $93.5$ & $98.5$ \\
    \midrule
    PUMA     & Standard & 50.8 & 79.6 \\
    RemeDi   & RemeDi   & 50.2 & 84.3 \\
    PRISM    & PRISM    & 52.2 & 86.0 \\
    BackPlay & BackPlay & 51.9 & 87.9 \\
    ProSeCo  & ProSeCo & 55.7 & 82.2 \\
    \midrule
    \multirow{4}{*}{TCT (ours)}
        & Standard                   & $51.2$          & $79.7$          \\
        & ReMDM (correction)         & $92.3$          & $96.5$          \\
        & \rrd (correction, $M=1$)   & $43.7$          & $90.0$          \\
        & \rrd (correction) & $\mathbf{97.5}$ & $\mathbf{99.3}$ \\
    \bottomrule
  \end{tabular}
\end{table}

\Cref{tab:main-comparison} compares the accuracy against prior work on Sudoku-Hard and 3-SAT. The per-level correction budget $M$ is set to 256 for Sudoku-Hard and 32 for 3-SAT. We use a smaller budget for 3-SAT because performance almost saturates at $M=32$. ReMDM and ProSeCo use the same compute budget for a fair comparison.
\rrd with standard training already surpasses other baselines on both tasks. TCT training further pushes the limit to 97.5\% on Sudoku-Hard and 99.3\% on 3-SAT. On the same TCT checkpoint and at the same budget, \rrd outperforms ReMDM (correction) by $5.2$ and $2.8$ points, so the gain comes from the remask--denoise kernel and not from the training alone. Furthermore, 3-SAT has higher peak performance than Sudoku-Hard, which aligns with our task structure analysis in \Cref{sec:task}.

\subsection{Coding Benchmark}
\label{sec:humaneval}
\begin{wrapfigure}{r}{0.5\linewidth}
  \centering
  \includegraphics[width=\linewidth]{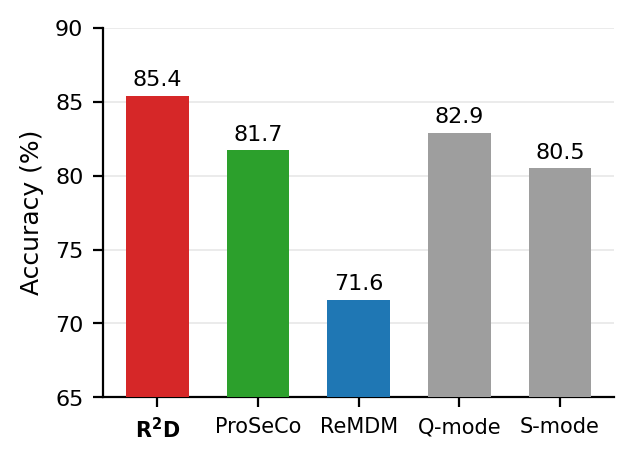}
  \caption{HumanEval+ accuracy with the off-the-shelf LLaDA2.1-mini (16B). The three
  colored bars are inference methods we run ourselves, and the two
  grey bars are the model's own published Q-mode and S-mode results. \rrd achieves the top accuracy, surpassing the official Q-mode performance.}
  \label{fig:humaneval-bars}
\end{wrapfigure}

To examine how \rrd behaves on practical language tasks with off-the-shelf models, we further evaluate on the coding benchmark HumanEval+ \citep{liuYourCodeGenerated2023} with LLaDA2.1-mini \citep{bieLLaDA21SpeedingText2026}, a 16B diffusion language model already trained with self-correction ability. \Cref{fig:humaneval-bars} reports the results. Besides the three iterative correction methods we run ourselves, we also quote the model's own published Q-mode and S-mode decoding results. \rrd reaches $85.4\%$, the top accuracy even beyond the model's published Q-mode accuracy ($82.9\%$), and it surpasses ProSeCo and ReMDM, two iterative correction baselines.

\subsection{Test-time Scaling}
\label{sec:exp-scaling}
The next question is how test-time performance scales with compute budget, so we compare the accuracy of \rrd, ReMDM and ProSeCo against the per-level correction budget $M$ on logical tasks. \rrd and ReMDM are evaluated with standard training checkpoints, while ProSeCo uses its own training method. As shown in \Cref{fig:main-scaling}, \rrd scales near-log-linearly: From $19.7\%$ at $M=1$ to $93.5\%$ at $M=256$ on Sudoku-Hard, and from $86.0\%$ to $98.5\%$ at $M=32$ on 3-SAT. ReMDM also improves with the budget but stays below \rrd at every budget. ProSeCo plateaus early at low performance: Its correction steps do not remask ($t=0$), where $\rho_0(p)=1$ and \Cref{prop:identity} guarantees no contraction.

\subsection{Compute Comparison}
\label{sec:exp-compute}

\begin{figure}[t]
  \centering
  \includegraphics[width=0.8\linewidth]{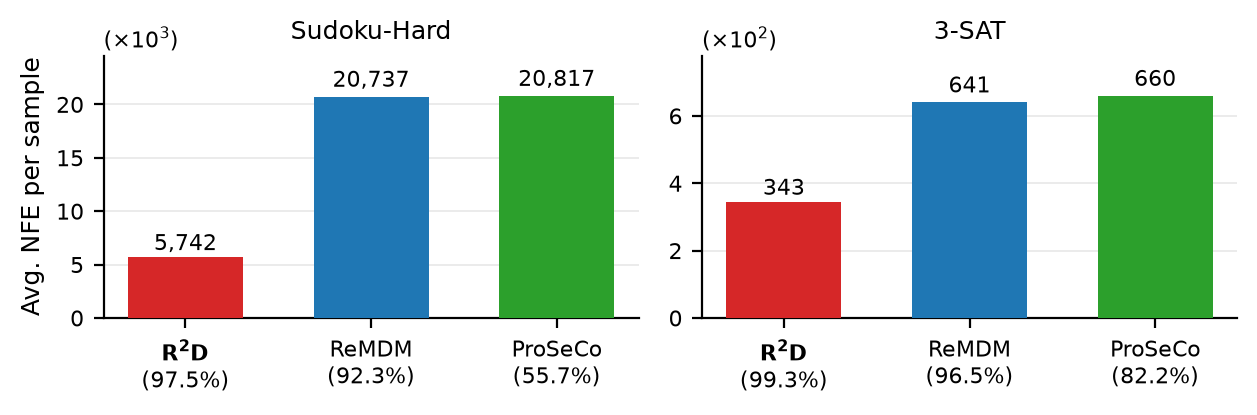}
  \caption{Averaged number of function evaluations (NFE) per sample at the budget of \Cref{tab:main-comparison} (lower is better), with accuracy in parentheses (TCT checkpoint for \rrd and ReMDM). \rrd exits early, whereas ReMDM and ProSeCo spend their full budget.}
  \label{fig:nfe-bars}
\end{figure}

Iterative correction methods buy accuracy with more model forward passes, so a fair comparison should account for compute as well as accuracy. \Cref{fig:nfe-bars} reports the averaged number of function evaluations (NFE) per sample that each method spends at the budget of \Cref{tab:main-comparison}. \rrd and ReMDM use the TCT-trained checkpoint. ReMDM and ProSeCo have no stopping rule and spend their entire budget. \rrd applies early exit (Line 15 of \Cref{alg:r2d-inference}) as soon as its prediction $\hat{\vec x}_0$ stabilizes, so it spends only about $1/4$ and $1/2$ of ReMDM's compute and reaches higher accuracy.

\subsection{Ablation Study}
\paragraph{\rrd design.} We ablate our inference approach on three critical designs: Uniform
remasking, early exit, and $\arg\max$ decoding. For uniform remasking, we replace it with low-confidence
remasking, a popular strategy used in standard inference. For early exit, we remove it
as ablation. For $\arg\max$ decoding (Line 6 of \Cref{alg:r2d-inference}), we replace it with stochastic sampling from the denoiser's raw distribution. \Cref{tab:remasking-ablation} (\Cref{sec:app_remasking}) reports the accuracy on Sudoku-Hard for each configuration across different compute budgets.
Low-confidence remasking almost removes the scaling effect: Accuracy moves only from $68.7\%$ to
$69.6\%$ across a $16\times$ increase in budget, where uniform remasking goes from $65.8\%$ to
$97.5\%$ over the same sweep. Removing early exit preserves the scaling trend but pays heavily for it: At $M=256$ the NFE is nearly 4 times that of the default configuration. Moreover, its accuracy is still lower than that of the default configuration, because random remasking can move the sequence away from the stationary point after convergence. Stochastic sampling
loses $10.4$ points on accuracy because
TCT trains the denoiser along an $\arg\max$ trajectory. It also spends more NFE, as a stochastic prediction makes it harder to fulfill the early-exit test.

\paragraph{TCT design.} We ablate TCT on its two designs in \Cref{alg:correct-training}: The top-$k$ unmasking position selection (Line 11), and the $\arg\max$ sampling (Line 9). For position selection, we replace top-$k$ with a random choice among the masked positions; For the reveal operator, we replace $\arg\max$ with temperature-$1$ sampling. We report the accuracy on Sudoku-Hard under \rrd inference, using compute budget $M=256$, as shown in \Cref{tab:tct-reveal-ablation} (\Cref{sec:app_reveal}).
Position selection is the load-bearing choice: Revealing random positions costs $12.0$ points, dropping accuracy from $97.5\%$ to $85.5\%$, because top-$k$ keeps the training trajectory close to a valid partial solution rather than letting it degrade into a corrupted one. The reveal operator is a minor lever by comparison, costing $0.7$ points, consistent with its role of matching the $\arg\max$ decoding that \rrd uses at inference.

\section{Conclusion}

We presented a unified Markov-chain view of masked diffusion inference in which standard decoding and self-correction arise from a single remask--denoise kernel acting on the clean sequence space. The contraction analysis motivates Remask-Denoise (\rrd) inference and Trajectory Correction Training (TCT), which together produce a nearly log-linear scaling effect on Sudoku-Hard and 3-SAT. Our inference approach also works with pretrained diffusion models on practical language tasks, surpassing published Q-mode accuracy on HumanEval+ with LLaDA2.1-mini. In addition, \rrd adopts an early-exit mechanism and is thus more compute-efficient. Applying \rrd and TCT to other domains, such as biological sequence design, is a promising future direction.

\subsubsection*{Acknowledgments}
We acknowledge support by National Science Foundation under Grant CCF-2106778. This work used computational resources provided through ACCESS allocation CIS260480, which is supported by National Science Foundation grants \#2138259, \#2138286, \#2138307, \#2137603, and \#2138296. Zhaoyi Zhou was additionally supported by the Wei Shen and Xuehong Zhang Presidential Fellowship and the Marija Ilić Endowed Fellowship at Carnegie Mellon University.

\bibliographystyle{iclr2027_conference}
\bibliography{ref,ref_more}

\appendix

\section{Proofs}
\label{sec:proofs}

\subsection{Proof for \Cref{prop:stationarity}}
\label{app:stationarity-proof}
\begin{proof}
By the definition of $K_t$,
\begin{align*}
(\pdata K_t)(\vec x_0')
&= \sum_{\vec x_0} \pdata(\vec x_0) \sum_{\vec x_t} q_{0\mid t}(\vec x_0' \mid \vec x_t)\, q_{t\mid 0}(\vec x_t \mid \vec x_0) \\
&= \sum_{\vec x_t} q_{0\mid t}(\vec x_0' \mid \vec x_t) \sum_{\vec x_0} q_{t\mid 0}(\vec x_t \mid \vec x_0)\, \pdata(\vec x_0) \\
&= \sum_{\vec x_t} q_{0\mid t}(\vec x_0' \mid \vec x_t)\, q_t(\vec x_t) \;=\; \pdata(\vec x_0'),
\end{align*}
where the last step applies Bayes' rule $q_{0\mid t}(\vec x_0' \mid \vec x_t)\, q_t(\vec x_t) = q_{t\mid 0}(\vec x_t \mid \vec x_0')\, \pdata(\vec x_0')$ and then sums $q_{t\mid 0}(\vec x_t \mid \vec x_0')$ over $\vec x_t$ to $1$.
\end{proof}

\subsection{Proof for \Cref{prop:identity}}
\label{app:identity-proof}
\begin{proof}
We first derive the stronger equality
\begin{equation}\label{eq:identity-equality}
\KL(\pdata \,\|\, p') \;=\; \rho_t(p)\,\KL(\pdata \,\|\, p) \;+\; \epsilon_t \;-\; R_t,
\end{equation}
where $R_t \ge 0$ is the remainder defined below; The bound in \Cref{prop:identity} then follows by dropping $R_t$. Define joint distributions over $(\vec x_t, \vec x_0')$
\begin{align*}
\mu(\vec x_t, \vec x_0') \;&=\; \tilde q_t(\vec x_t)\, f_\theta(\vec x_0' \mid \vec x_t), &
\nu(\vec x_t, \vec x_0') \;&=\; q_t(\vec x_t)\, q_{0\mid t}(\vec x_0' \mid \vec x_t).
\end{align*}
The $\vec x_0'$-marginal of $\mu$ is $p'$ (by definition of $p' = p \tilde K^\theta_t$), and the $\vec x_0'$-marginal of $\nu$ is $\pdata$ (by Bayes' rule, $q_t(\vec x_t)\, q_{0\mid t}(\vec x_0' \mid \vec x_t) = \pdata(\vec x_0')\, q_{t\mid 0}(\vec x_t \mid \vec x_0')$).

Expanding $\KL(\nu \,\|\, \mu)$ via the factorization $\mu = \tilde q_t \cdot f_\theta$ and $\nu = q_t \cdot q_{0\mid t}$ gives
\[
\KL(\nu \,\|\, \mu) \;=\; \KL(q_t \,\|\, \tilde q_t) \;+\; \epsilon_t \;=\; \rho_t(p)\,\KL(\pdata \,\|\, p) \;+\; \epsilon_t,
\]
using \eqref{eq:rho}. Expanding the same divergence via the alternative factorization $\nu(\vec x_t, \vec x_0') = \pdata(\vec x_0')\, q_{t\mid 0}(\vec x_t \mid \vec x_0')$ yields
\[
\KL(\nu \,\|\, \mu) \;=\; \KL(\pdata \,\|\, p') \;+\; R_t,
\]
where
\[
R_t \;=\; \E_{\vec x_0' \sim \pdata,\, \vec x_t \sim q_{t\mid 0}(\cdot \mid \vec x_0')}\!\left[\log \frac{q_{t\mid 0}(\vec x_t \mid \vec x_0')\, p'(\vec x_0')}{\tilde q_t(\vec x_t)\, f_\theta(\vec x_0' \mid \vec x_t)}\right] \;\ge\; 0,
\]
the non-negativity following because $R_t$ is itself a KL divergence. Equating the two expressions for $\KL(\nu \,\|\, \mu)$ gives \eqref{eq:identity-equality}; Dropping $R_t \ge 0$ yields the bound stated in \Cref{prop:identity}.
\end{proof}

\subsection{Proof for \Cref{prop:training-bound}}
\label{app:training-bound-proof}
\begin{proof}
Starting from the marginal of $p' = p \tilde K^\theta_t$ on $\vec x_0$,
\begin{align*}
p'(\vec x_0) \;=\; \sum_{\vec x_t} \tilde{q}_t(\vec x_t)\, f_\theta(\vec x_0 \mid \vec x_t) \;=\; \E_{\vec x_t \sim \tilde{q}_t}\!\bigl[f_\theta(\vec x_0 \mid \vec x_t)\bigr],
\end{align*}
apply Jensen's inequality directly:
\begin{align*}
-\log p'(\vec x_0) \;=\; -\log \E_{\vec x_t \sim \tilde{q}_t}\!\bigl[f_\theta(\vec x_0 \mid \vec x_t)\bigr] \;\le\; \E_{\vec x_t \sim \tilde{q}_t}\!\bigl[-\log f_\theta(\vec x_0 \mid \vec x_t)\bigr].
\end{align*}
Taking expectation over $\vec x_0 \sim \pdata$ and subtracting $H(\pdata)$ yields the claim.
\end{proof}

\section{Trajectory Correction Training Algorithm}
\label{app:tct-algorithm}
\Cref{alg:correct-training} gives the pseudo-code of a single TCT update step described in \Cref{sec:method}. The $B$ chains in a batch are advanced in a streaming fashion: After each gradient step, a chain that has reached the stage budget $K$ is refilled with a fresh sample and a fully masked sequence, while every other chain reveals the top-$k$ most confident masked positions with the model's $\arg\max$ predictions.

\begin{algorithm}[ht]
\caption{Trajectory Correction Training (Single Update Step)}
\label{alg:correct-training}
\begin{algorithmic}[1]
\Require MDM $f_\param$, max stages $K$, batch of partially masked sequences $\{\tilde{\vec x}^{(j)}\}_{j=1}^B$, paired clean samples $\{\vec x_0^{(j)}\}_{j=1}^B$, per-example stage counters $\{n^{(j)}\}_{j=1}^B \in [0, K]$
\State Compute logits $\{f_\param(\cdot \mid \tilde{\vec x}^{(j)})\}_{j=1}^B$
\State Compute correction loss $L_{\mathrm{correct}}(\param)$ (\Cref{eq:correction-loss}) from logits and $\{\vec x_0^{(j)}\}$; Update $\param$
\For{$j = 1, \cdots, B$ \textbf{in parallel}} \Comment{Streaming update}
    \If{$n^{(j)} = K$}
        \State Sample $\vec x_0^{(j)} \sim \pdata$
        \State $n^{(j)} \gets 0$
        \State Initialize $\tilde{\vec x}^{(j)} = \mask^L$. \Comment{refill: Restart chain on a new sequence}
    \Else
        \State $\hat x^{(j)}_i \gets \arg\max_{v \in V} f_\param^{\,i}(v \mid \tilde{\vec x}^{(j)})$,
        \State $s_i \gets \max_{v \in V} f_\param^{\,i}(v \mid \tilde{\vec x}^{(j)})$ for masked positions $i$ \Comment{unmasking policy from the model}
        \State Select reveal set $S^{(j)}$ from $\{s_i\}$ \Comment{top-$k$ with thresholding}
        \State Unmask $\tilde{\vec x}^{(j)}_i \gets \hat x^{(j)}_i$ for all $i \in S^{(j)}$;\quad $n^{(j)} \gets n^{(j)} + 1$ \Comment{fill with model predictions}
    \EndIf
\EndFor
\end{algorithmic}
\end{algorithm}

\section{Additional Experiment Results}
\label{sec:app_exp_result}

\subsection{Error bars of \rrd and ReMDM inference.}
\label{sec:app_scaling}
\Cref{fig:tct-scaling-overview} shows the error bars of \rrd and ReMDM, under both TCT and standard training. We use 3 random seeds for training, and 3 random seeds for inference on each training checkpoint. The full figure below places those panels beside their TCT-training counterparts and draws the spread across training seeds.

\begin{figure}[ht]
  \centering
  \begin{subfigure}[t]{0.49\linewidth}
    \centering
    \includegraphics[width=\linewidth]{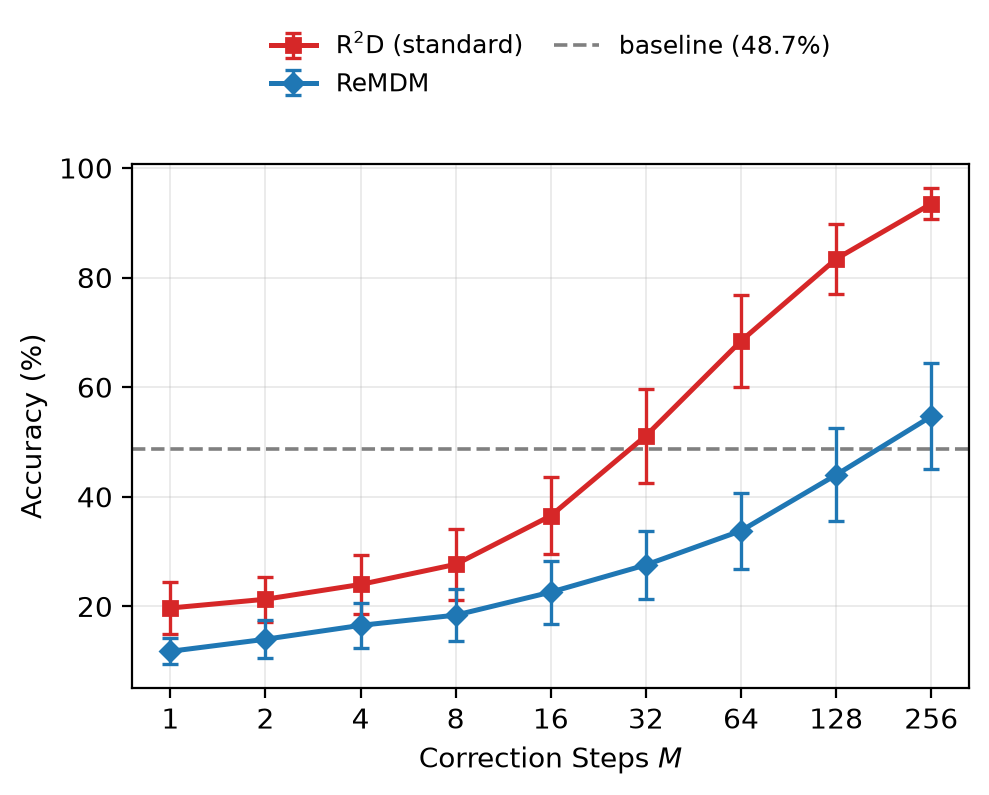}
    \caption{Sudoku-Hard, standard training}
  \end{subfigure}\hfill
  \begin{subfigure}[t]{0.49\linewidth}
    \centering
    \includegraphics[width=\linewidth]{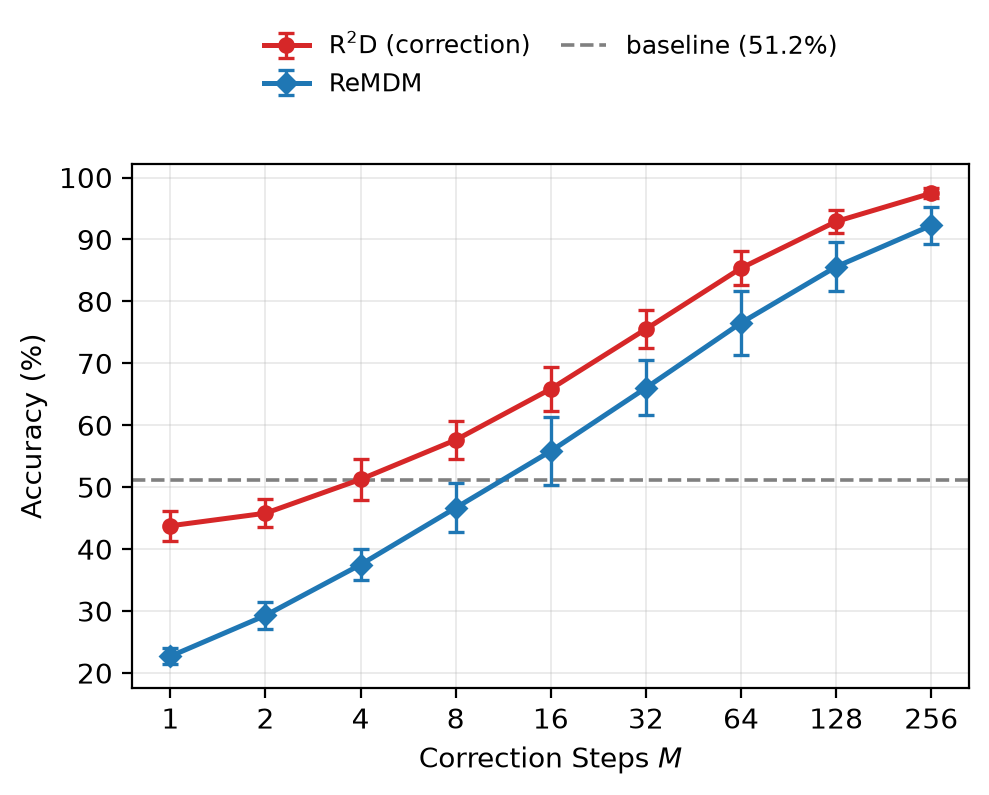}
    \caption{Sudoku-Hard, TCT training}
  \end{subfigure}

  \vspace{0.6em}
  \begin{subfigure}[t]{0.49\linewidth}
    \centering
    \includegraphics[width=\linewidth]{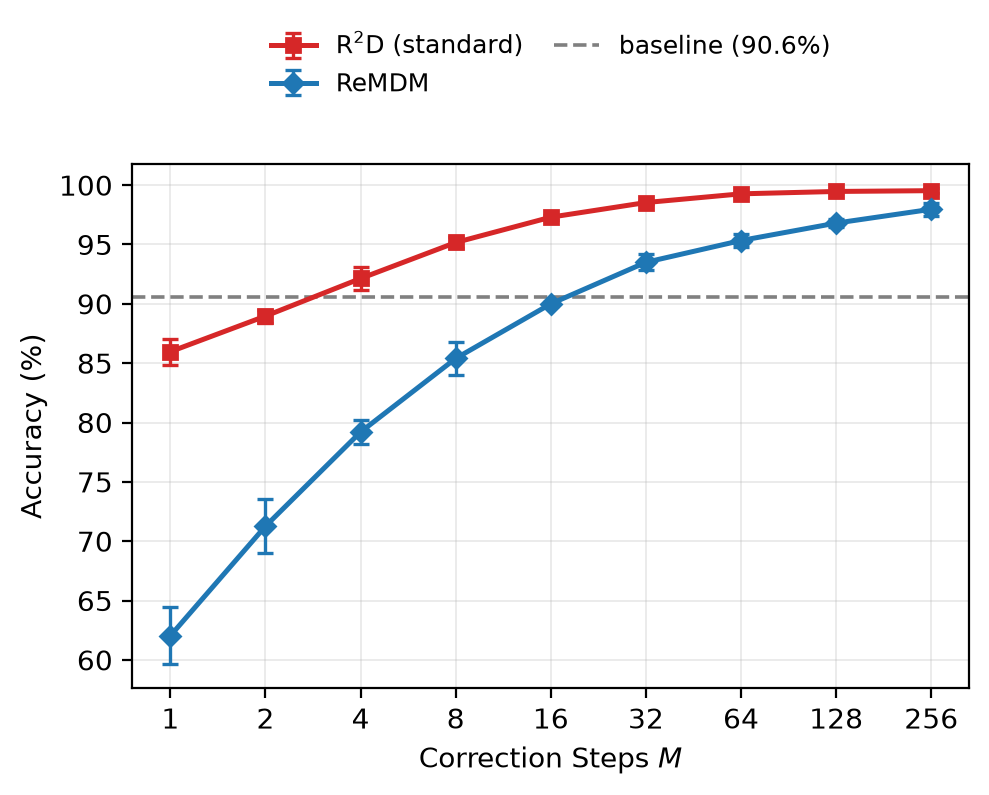}
    \caption{3-SAT, standard training}
  \end{subfigure}\hfill
  \begin{subfigure}[t]{0.49\linewidth}
    \centering
    \includegraphics[width=\linewidth]{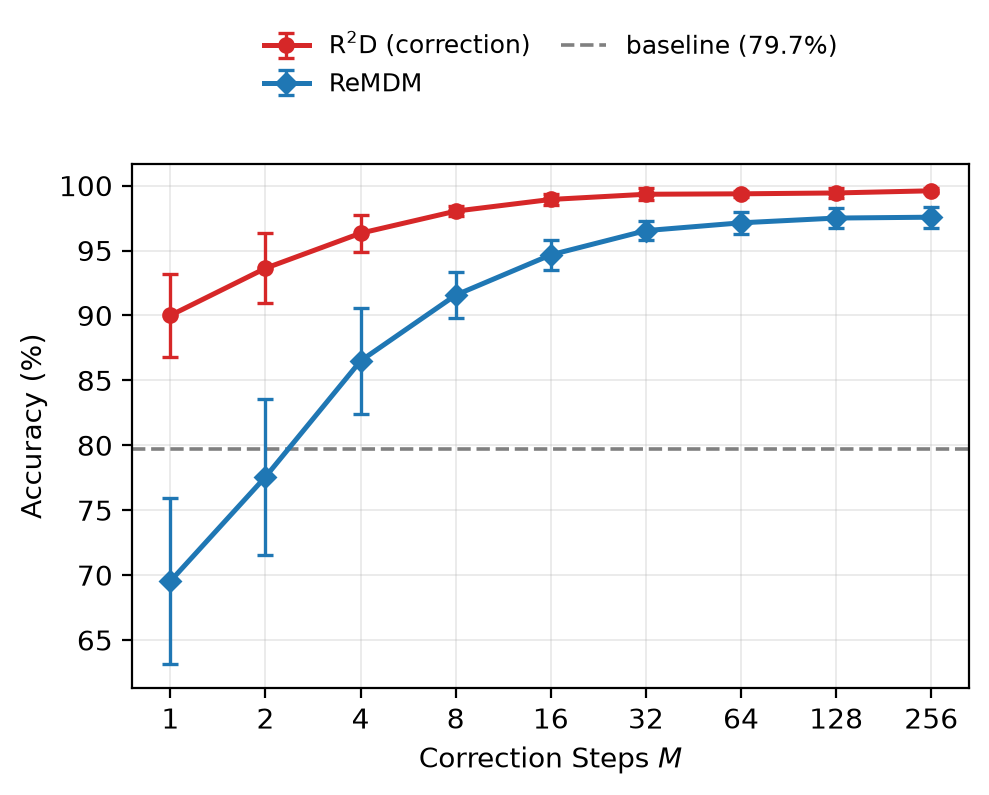}
    \caption{3-SAT, TCT training}
  \end{subfigure}
  \caption{Inference scaling on Sudoku-Hard (top) and 3-SAT (bottom), under standard training (left) and TCT training (right). The y-axis reports task accuracy and the x-axis the per-token correction budget $M$ on a logarithmic scale; error bars are one standard deviation across three training seeds. The horizontal dashed line marks each panel's own no-correction baseline. \rrd scales near-log-linearly in all four panels and stays above ReMDM at every budget. Protocol details are in \Cref{sec:app_exp_detail}.}
  \label{fig:tct-scaling-overview}
\end{figure}

\subsection{Compute and Latency}
\label{sec:app_compute}
We report the averaged NFE and averaged wall time of ReMDM, \rrd and ProSeCo.
\Cref{tab:compute-normalized} is the compute-normalized view of \Cref{tab:main-comparison},
discussed in \Cref{sec:exp-compute}. We also report standard deviation of the accuracy for ReMDM and \rrd, under 3 random seeds. Wall time is measured at batch size $1$ on a single NVIDIA GH200.
The cost of correction relative to standard decoding is nonetheless real. At the largest budget, \rrd's per-sample latency is about $77\times$ that of standard decoding on the same checkpoint for Sudoku-Hard ($49.4$ s against $0.64$ s) and about $26\times$ on 3-SAT ($4.76$ s against $0.18$ s). What makes this practical is that the budget $M$ is a free knob rather than a fixed cost: At $M = 1$ the same sampler costs $1.31$ s and $0.36$ s respectively, within roughly twice the standard baseline, and every point in between is available along the curves of \Cref{fig:main-scaling}.

\begin{table}[ht]
  \centering
  \caption{Compute-normalized comparison on Sudoku-Hard and 3-SAT. Each inference method is reported with its accuracy (in $\%$), averaged NFE, and per-sample wall time at batch size $1$ on a single NVIDIA GH200 GPU. \rrd is the only method that exits early: on Sudoku-Hard it reaches $\mathbf{97.5}$ in $49.4$ s/sample where ReMDM reaches $92.3$ in $187.9$ s/sample --- $5.2$ points better at $3.8\times$ less wall time. Measurement protocol in \Cref{sec:app_exp_detail}.}
  \label{tab:compute-normalized}
  \small
  \setlength{\tabcolsep}{4pt}
  \begin{tabular}{ll ccc ccc}
    \toprule
    & & \multicolumn{3}{c}{Sudoku-Hard} & \multicolumn{3}{c}{3-SAT} \\
    \cmidrule(lr){3-5}\cmidrule(lr){6-8}
    Training method & Inference method & Acc. & NFE & s/sample & Acc. & NFE & s/sample \\
    \midrule
    \multirow{4}{*}{Standard}
        & Standard                       & $48.7\std{1.0}$ & 81    & 0.74  & $90.6\std{2.3}$ & 20  & 0.16 \\
        & ReMDM (standard)               & $54.7\std{9.7}$ & 20737 & 177.2 & $93.5\std{0.7}$ & 641 & 5.30 \\
        & \rrd (standard, $M=1$)         & $19.7\std{4.7}$ & 162   & 1.41  & $86.0\std{1.1}$ & 40  & 0.42 \\
        & \rrd (standard, $M=M^\dagger$) & $93.5\std{2.8}$ & 4607  & 36.9  & $98.5\std{0.2}$ & 336 & 2.70 \\
    \midrule
    ProSeCo & ProSeCo ($M=M^\dagger$) & 55.7 & 20817 & 168.5 & 82.2 & 660 & 5.02 \\
    \midrule
    \multirow{4}{*}{TCT (ours)}
        & Standard                         & $51.2\std{1.5}$          & 81    & 0.64  & $79.7\std{2.1}$          & 20  & 0.18 \\
        & ReMDM (correction)               & $92.3\std{2.9}$          & 20737 & 187.9 & $96.5\std{0.8}$          & 641 & 4.95 \\
        & \rrd (correction, $M=1$)         & $43.7\std{2.4}$          & 162   & 1.31  & $90.0\std{3.2}$          & 40  & 0.36 \\
        & \rrd (correction, $M=M^\dagger$) & $\mathbf{97.5}\std{0.8}$ & 5742  & 49.4  & $\mathbf{99.3}\std{0.5}$ & 343 & 4.76 \\
    \bottomrule
  \end{tabular}
\end{table}

\subsection{Compute and Latency on HumanEval+}
\label{sec:app_humaneval_compute}
\Cref{tab:humaneval-cost} pairs the accuracy of \Cref{fig:humaneval-bars} with the cost of the
three correction methods. \rrd reaches top accuracy with roughly half of ReMDM's compute, resulting from the early-exit mechanism.

\begin{table}[ht]
  \centering
  \caption{HumanEval+ accuracy, averaged NFE and averaged wall time (in \emph{minute} per sample)
  for the three iterative correction methods of \Cref{fig:humaneval-bars}, all evaluated with the
  off-the-shelf LLaDA2.1-mini. \rrd reaches the top accuracy while spending roughly half
  of ReMDM's compute.}
  \label{tab:humaneval-cost}
  \begin{tabular}{lccc}
    \toprule
    Method & Accuracy & Avg NFE & Avg wall time \\
    \midrule
    \rrd (correction) & $\mathbf{85.4}$ & 757  & 6.6  \\
    ProSeCo   & 81.7                     & 692  & 6.0  \\
    ReMDM     & 71.6                     & 1632 & 14.3  \\
    \bottomrule
  \end{tabular}
\end{table}

\subsection{ReMDM Implementation Choice under TCT Training}
\label{sec:app_remdm}
We stick to the official implementation from \url{https://github.com/kuleshov-group/remdm} when running ReMDM with standard training. When the model is trained with TCT, we make one adaptation to ReMDM inference: We use argmax sampling instead of the categorical stochastic sampling in the official implementation.

We actually consider two critical implementation choices for ReMDM: Carry-over unmasking vs. no carry-over, and argmax sampling against stochastic sampling. \Cref{tab:remdm-variants} compares the 4 configurations on both Sudoku-Hard and 3-SAT. Each entry is evaluated on 9 repeated runs: three TCT training checkpoints with different seeds, and three inference seeds for each checkpoint. Stochastic sampling is significantly weaker on both tasks. This is because TCT models are trained on argmax sampling trajectories, and thus stochastically revealed tokens are out-of-distribution for these models. Under argmax sampling, Sudoku-Hard is indifferent to the unmasking choice, but 3-SAT favors carry-over unmasking.
Therefore, we report the accuracy of ReMDM under carry-over unmasking + argmax sampling in \Cref{tab:main-comparison} to ensure a fair comparison with \rrd, and refer to this configuration as \textbf{ReMDM (correction)}. The official configuration, i.e., carry-over unmasking + stochastic sampling, is referred to as ReMDM (standard) and is still used for standard-trained models.

\begin{table}[ht]
  \centering
  \caption{ReMDM implementation variants on the TCT checkpoint, comparing carry-over denoising against the sampling operator, at $M = 256$ on Sudoku-Hard and $M = 32$ on 3-SAT. Accuracy is in $\%$, mean $\pm$ standard deviation over three training seeds $\times$ three inference seeds. The sampling operator dominates: $\arg\max$ is the operator TCT trains under. Carry-over is close to free on Sudoku-Hard but decisive on 3-SAT. The standard checkpoint is not ablated; there we run the official carry-over $+$ stochastic configuration.}
  \label{tab:remdm-variants}
  \begin{tabular}{llcc}
    \toprule
    Carry-over & Sampling & Sudoku-Hard & 3-SAT \\
    \midrule
    Yes & Stochastic & $40.0\std{6.7}$          & $9.9\std{3.0}$           \\
    No  & Stochastic & $46.9\std{12.3}$         & $1.4\std{1.8}$           \\
    Yes & Argmax     & $92.3\std{2.9}$          & $\mathbf{96.5}\std{0.8}$ \\
    No  & Argmax     & $\mathbf{92.4}\std{2.6}$ & $67.7\std{30.9}$         \\
    \bottomrule
  \end{tabular}
\end{table}

We further investigate why the two unmasking choices (carry-over vs. no carry-over) have different
impacts on the two tasks. We find that the TCT-trained model on Sudoku holds a correct grid
exactly in place, whereas on 3-SAT it keeps rewriting a small fraction of an already-correct
assignment. To test this, we feed the TCT-trained model the ground-truth answer, randomly mask a
fraction of the answer tokens, run a single forward pass \emph{without} carry-over, and measure the
fraction of revealed tokens the model reproduces unchanged. As shown in \Cref{tab:token-preservation},
on Sudoku-Hard the
model preserves every correct revealed cell at every masking ratio. On 3-SAT, however, preservation is exact at high mask ratios but decreases at smaller ones. In other words, the model tends to modify the correct tokens at a small probability when most of the tokens are revealed in the sequence. Therefore, a correct solution can be modified by mistake near the end of generation, and this explains the high variance on 3-SAT without carry-over.

\begin{table}[ht]
  \centering
  \caption{Token preservation of the TCT-trained denoiser: The fraction of correct revealed tokens that survive a single forward pass without carry-over, averaged over three training seeds. On Sudoku-Hard the correct grid is an exact fixed point of the denoiser at every masking ratio; On 3-SAT it is not, and preservation is weakest precisely when the sequence is nearly complete, which is why carry-over is essential there and not on Sudoku-Hard.}
  \label{tab:token-preservation}
  \begin{tabular}{ccc}
    \toprule
    Mask ratio & Sudoku-Hard & 3-SAT \\
    \midrule
    0.00 & 1.000 & 0.951 \\
    0.10 & 1.000 & 0.988 \\
    0.30 & 1.000 & 0.997 \\
    0.50 & 1.000 & 0.998 \\
    0.70 & 1.000 & 1.000 \\
    0.90 & 1.000 & 1.000 \\
    0.95 & 1.000 & 1.000 \\
    \bottomrule
  \end{tabular}
\end{table}

\section{Ablation Study}

\subsection{\texorpdfstring{\rrd}{R2D} Inference Design}
\label{sec:app_remasking}
This section ablates the three elements in \rrd design. Starting from the default configuration, each row changes one design element of
\Cref{alg:r2d-inference} in isolation: Uniform remasking is replaced by low-confidence remasking,
early exit is disabled, and the $\arg\max$ decoding of Line 6 is replaced by sampling from the
denoiser. \Cref{tab:remasking-ablation} gives the accuracy result along with averaged NFE on Sudoku-Hard with TCT training, under different compute budgets. We use three model checkpoints trained with three different random seeds. Each table entry is measured on 9 repeated evaluations, three runs for each checkpoint.

Low-confidence remasking almost removes the scaling effect: Accuracy moves only from $68.7\%$ to
$69.6\%$ across a $16\times$ increase in budget, where uniform remasking goes from $65.8\%$ to
$97.5\%$ over the same sweep. Removing early exit preserves the scaling trend but pays heavily for it: At $M=256$ the NFE is nearly 4 times that of the default configuration. Moreover, its accuracy is still lower than that of the default configuration, because random remasking can move the sequence away from the stationary point after convergence. Stochastic sampling
loses $10.4$ points on accuracy because
TCT trains the denoiser along an $\arg\max$ trajectory. It also spends more NFE, as a stochastic prediction makes it harder to fulfill the early-exit test.

\begin{table}[ht]
  \centering
  \caption{\rrd (correction) inference design on Sudoku-Hard under TCT training. Each row changes one design element of \Cref{alg:r2d-inference} relative to \texttt{Default}. Each cell reports accuracy (in $\%$, mean $\pm$ standard deviation over three training seeds) with averaged NFE in parentheses. Low-confidence remasking almost removes the scaling effect; Removing early exit preserves scaling but at markedly higher NFE; Stochastic sampling in place of $\arg\max$ costs accuracy at every budget while spending more NFE.}
  \label{tab:remasking-ablation}
  \footnotesize
  \setlength{\tabcolsep}{1.5pt}
  \begin{tabular}{lccccc}
    \toprule
    Run & $M=16$ & $M=32$ & $M=64$ & $M=128$ & $M=256$ \\
    \midrule
    \texttt{Default}        & $65.8\std{3.6}$ (1050)          & $\mathbf{75.6}\std{3.1}$ (1477) & $\mathbf{85.4}\std{2.8}$ (2412) & $\mathbf{92.9}\std{1.9}$ (3552)  & $\mathbf{97.5}\std{0.8}$ (5742)  \\
    \texttt{Low confidence} & $\mathbf{68.7}\std{1.8}$ (751)  & $69.0\std{2.4}$ (990)           & $69.5\std{2.4}$ (1356)          & $69.3\std{2.4}$ (2212)           & $69.6\std{2.3}$ (3271)           \\
    \texttt{No early exit}  & $63.9\std{3.1}$ (1377)          & $71.7\std{3.3}$ (2673)          & $79.2\std{4.1}$ (5265)          & $86.7\std{3.9}$ (10449)          & $93.3\std{2.4}$ (20817)          \\
    \texttt{Stochastic}     & $37.5\std{4.3}$ (1181)          & $50.1\std{6.3}$ (2034)          & $63.6\std{6.0}$ (3005)          & $76.1\std{6.3}$ (4146)           & $87.1\std{4.7}$ (6156)           \\
    \bottomrule
  \end{tabular}
\end{table}

\subsection{TCT Training Design}
\label{sec:app_reveal}
The $\arg\max$ reveal of Line 9 and the top-$k$ confidence reveal of Line 11 in
\Cref{alg:correct-training} are deliberate design choices rather than arbitrary deviations from the
theory. Following PUMA \citep{kimStopTrainingWorst2026}, TCT trains the denoiser on its own
inference-time trajectory; since \rrd decodes with $\arg\max$, revealing with $\arg\max$ during
training keeps the training and inference operators consistent, and revealing the highest-confidence
positions keeps the training trajectory close to correct rather than letting it degrade into noise.

\Cref{tab:tct-reveal-ablation} ablates each choice independently on Sudoku-Hard, replacing the
$\arg\max$ reveal with temperature-$1$ sampling and, separately, the top-$k$ reveal with a random
choice among the masked positions. Position selection is the load-bearing choice: Random reveal costs
$12.0$ points. The reason is that top-$k$ keeps the revealed prefix mostly correct, so the model
spends its capacity learning to predict the remaining masks from a valid partial solution --- that is,
learning the task structure --- rather than learning to complete corrupted grids. The reveal operator
is a minor lever by comparison: Stochastic reveal costs only $0.7$ points, consistent with its role as
the operator that \rrd inference matches. Both approximations therefore help rather than hurt.

\begin{table}[ht]
  \centering
  \caption{Reveal rule in TCT training (\Cref{alg:correct-training}) on Sudoku-Hard, replacing the argmax reveal of Line 9 and the top-$k$ confidence reveal of Line 11. Position selection is the load-bearing choice, while argmax reveal is a minor lever; Both approximations help rather than hurt.}
  \label{tab:tct-reveal-ablation}
  \begin{tabular}{lc}
    \toprule
    Reveal rule & Sudoku-Hard acc. ($\%$) \\
    \midrule
    Top-$k$ confidence $+$ argmax (published) & \textbf{97.5} \\
    Random position $+$ argmax (Line 11 ablation) & 85.5 \\
    Top-$k$ confidence $+$ stochastic (Line 9 ablation) & 96.8 \\
    \bottomrule
  \end{tabular}
\end{table}

\section{Additional Experiment Details}
\label{sec:app_exp_detail}
\subsection{Sudoku-Hard}
\paragraph{Data link.} \url{https://huggingface.co/datasets/sapientinc/sudoku-extreme}.

\paragraph{Model architecture.} The denoiser is a small Qwen2-style decoder-only transformer trained from scratch. The full architecture is summarized in \Cref{tab:denoiser-arch} (column ``Sudoku-Hard''); The resulting parameter count is approximately $6.8$M, of which the embedding and language-model head contribute only $\sim 5.6$K parameters due to the tiny vocabulary (digits $1$--$9$, the mask token, and a padding token).

\begin{table}[ht]
  \centering
  \caption{Denoiser architecture for Sudoku-Hard and 3-SAT. The two share a Qwen2-style backbone and differ only in vocabulary size and maximum position, which are dictated by the input encoding.}
  \label{tab:denoiser-arch}
  \begin{tabular}{lcc}
    \toprule
    Hyperparameter & Sudoku-Hard & 3-SAT \\
    \midrule
    Vocabulary size & 11 & 42 \\
    Hidden size & 256 & 256 \\
    FFN intermediate size (SwiGLU) & 768 & 768 \\
    Number of layers & 8 & 8 \\
    Number of attention heads & 8 & 8 \\
    Number of key/value heads & 8 & 8 \\
    Maximum position & 162 & 657 \\
    Dropout & 0.0 & 0.0 \\
    \midrule
    Total parameters & $\approx 6.8$M & $\approx 6.8$M \\
    \bottomrule
  \end{tabular}
\end{table}

\paragraph{Training.} The denoiser is trained with TCT (\Cref{alg:correct-training}) on the Sudoku-Hard training split. We use a per-GPU batch size of $256$. The checkpoint reported in the main paper is taken at step $460{,}000$. A single training run takes around 3 hours on a single Nvidia L40 GPU.

\paragraph{Inference.} We set correction frequency to 1 and early-exit patience to $s = 5$ consecutive unchanged $\hat{\vec x}_0$, run the inner loop at every grid point (no skipped large-$t$ steps), and disable early commitment of revealed tokens so that the kernel may freely revise any position throughout the trajectory.

\subsection{3-SAT}
\paragraph{Data link.} \url{https://www.cs.ubc.ca/~hoos/SATLIB/benchm.html}

\paragraph{Model architecture.} The 3-SAT denoiser shares the same Qwen2-style backbone as the Sudoku-Hard model and differs only in vocabulary size and maximum context length, both of which are dictated by the SATLIB encoding (20 variables, 91 clauses). See \Cref{tab:denoiser-arch} (column ``3-SAT''); The parameter count is again approximately $6.8$M.

\paragraph{Training.} The denoiser is trained from scratch with TCT (\Cref{alg:correct-training}) on a 1M-instance synthetic 3-SAT training set, using $K = 5$ correction phases per outer trajectory. The checkpoint reported in the main paper is taken at step $100{,}000$. A single training run takes around 4 hours on a single Nvidia GH200 GPU.

\paragraph{Inference.} As on Sudoku-Hard, we set correction frequency to 1 and early-exit patience to $s = 5$ consecutive unchanged $\hat{\vec x}_0$, run the inner loop at every grid point (no skipped large-$t$ steps), and disable early commitment of revealed tokens so that the kernel may freely revise any position throughout the trajectory.

\subsection{HumanEval+}
\label{sec:app_humaneval_detail}
Benchmark link can be found at \url{https://github.com/evalplus/evalplus}.

Unlike Sudoku-Hard and 3-SAT, this experiment is inference-only: We evaluate the off-the-shelf
LLaDA2.1-mini (16B MoE) \citep{bieLLaDA21SpeedingText2026} as released, with no TCT and no
finetuning of any kind, so what varies across methods is the inference algorithm alone. All $164$
HumanEval+ problems \citep{liuYourCodeGenerated2023} are scored with EvalPlus; generations are
produced on the base HumanEval prompts and re-scored against the plus test suite.

LLaDA2.1-mini decodes a response in blocks rather than over
the whole sequence at once: The response is split into consecutive blocks of $32$ tokens, up to
$512$ generated tokens ($16$ blocks), and the model attends to the prompt and all previously
decoded blocks while denoising only the current one. Within a block we decode greedily
(temperature $0$) with confidence-thresholded parallel unmasking: Every masked position whose
top-token probability exceeds $0.9$ is revealed at that step, with at least one position revealed
per step so that decoding always advances. Given the block-diffusion nature of the model, \rrd, ProSeCo and ReMDM all adopt a block-wise correction: Correction only happens on the current block, and prior blocks and future masks are left untouched. \rrd uses correction frequency $1$
(the kernel is applied at every block) and the early-exit patience is $\tau=2$. Remasking is
uniform at random, as in \Cref{alg:r2d-inference}. \rrd is reported as the mean $\pm$ standard
deviation over three inference seeds. \rrd and ReMDM both use 32 correction steps per token. ProSeCo uses 32 correction steps per block.

\end{document}